\documentclass[review]{elsarticle}

\usepackage{lineno,hyperref}
\modulolinenumbers[5]

\journal{https://arxiv.org}

\usepackage{times}
\usepackage{amssymb}
\usepackage{amsmath}
\usepackage{txfonts}
\usepackage{amssymb}
\usepackage{enumerate}
\usepackage{amsfonts}
\usepackage{times}
\usepackage{mathrsfs}
\usepackage{amscd}
\usepackage{stmaryrd}
\usepackage{graphicx}
\usepackage{makeidx}

\newcommand{\commentout}[1]{}

\newtheorem{theorem}{Theorem}
\newtheorem{example}{Example}
\newtheorem{propn}{Proposition}
\newtheorem{definition}{Definition}
\newenvironment{proof}{\noindent \textbf{Proof:}}{\hfill  $\boxempty$\\}
\newcommand{\be}{\begin{enumerate}} 
\newcommand{\ee}{\end{enumerate}}

\newcommand{\Ags}{Agt}
\newcommand{\Prop}{Prop}
\newcommand{\powerset}[1]{{\cal P}(#1) }

\newcommand{\until}{\,U\,}

\newcommand{\buchi}{B\"uchi}

\newcommand{\Acts}{\mathit{Acts}} 
\newcommand{\Fairness}{{\cal F}}

\newcommand{\nat}{\mathbb{N}}

\newcommand{\M}{K}
\newcommand{\subC}{B} 

\begin{document}
\begin{frontmatter}

\title{
Normative Multiagent Systems: A {\em Dynamic} Generalization
}

\author{Xiaowei Huang$^{1,2}$, Ji Ruan$^3$, Qingliang Chen$^1$, Kaile Su$^{1,4}$\\
$^1$Department of Computer Science, Jinan University, China\\
$^2$Department of Computer Science, University of Oxford, United Kingdom\\
$^3$School of Engineering, Computer and Mathematical Sciences, Auckland University of Technology, New Zealand\\
$^4$Institute for Integrated and Intelligent Systems, Griffith University, Brisbane, Australia}


\begin{abstract}
Social norms are powerful formalism in coordinating autonomous agents' behaviour to achieve certain objectives. In this paper, we propose a dynamic normative system to enable the reasoning of the changes of norms under different circumstances, which cannot be done in the existing static normative systems. 
We study two important problems (norm synthesis and norm recognition) related to the autonomy of the entire system and the agents, and characterise the computational complexities of solving these problems. 
\end{abstract}

\end{frontmatter}


\section{Introduction}

Multiagent systems have been used to model and analyse distributed and heterogeneous systems, with agents being suitable for  modelling  software processes and physical resources. 
%
%
%
Roughly speaking, autonomy means that the system by itself, or the agents in the system, can decide for themselves what to do and when to do it~\cite{FDW2013}. 
To facilitate autonomous behaviours, agents are provided with capabilities, e.g., to gather information by making observations (via e.g., sensors) and communicating with each other (via e.g., wireless network), to affect the environment and other agents by taking actions, etc. Moreover,  systems and agents may have specific objectives to pursue. 
In this paper, we study autonomy issues related to  social norms~\cite{ST1992}, which are powerful formalism for the coordination of agents, by restricting their behaviour to prevent destructive interactions from taking place, or to facilitate positive interactions~\cite{vdHRW2007,AW2010}. 
%
%
%

%
Existing normative systems~\cite{WvdH2005,AvdHW2007,CR2009,AW2010,MSRWV2013} impose restriction rules on the multiagent systems to disallow agents' actions based on the evaluation of  current system state. An implicit assumption behind this setting is that the normative systems do not have (normative) states to describe different social norms under different circumstances. That is, they are {\it static} normative systems. Specifically, if an action is disallowed on some system state then it will remain disallowed when the same system state occurs again. However, more realistically, social norms may be subject to changes. For example, a human society has different social norms in peacetime and wartime, and an autonomous multiagent system may have different social norms when exposed to different levels of cyber-attacks. This motivates us to propose a new definition of normative systems (in Section \ref{sec:normsys}), to enable the representation of norms under multiple states (hence, {\em dynamic} normative systems). With a running example, we show that a dynamic normative system can be a necessity if a multiagent system wants to implement certain objectives. 




We focus on two related autonomy issues\footnote{In this paper, we consider decision problems of these autonomy issues. The algorithms for the upper bounds in Theorem~\ref{thm:synthesis}, \ref{thm:NC_1}, and \ref{thm:NC_2} can be adapted to implement their related autonomy.}.
The first is  norm synthesis, which is to determine the existence of a normative system for the achievement of objectives. The success of this problem suggests the autonomy of the multiagent system with respect to the objectives, i.e., if all agents in the system choose to conform to the normative system\footnote{The synthesised normative system is a common knowledge~\cite{FHMVbook} to the agents. }, the objectives can be achieved.  For static normative systems, norm synthesis problem is shown to be NP-complete~\cite{ST1995}. For our new, and more general, definition of  normative systems, we show that it is EXPTIME-complete. This encouraging decidable result shows that the maximum number of normative states can be bounded. 

The second is norm recognition, which can be seen as a successive step after deploying an autonomous multiagent systems (e.g., by norm synthesis). For deployed  systems such as~\cite{CGKLOPRT2001}, it can be essential to allow new agents to join anytime. If so, it is generally expected that the new agent is able to recognise the current social norms after  playing in the system for a while. 
Under this general description, we consider two subproblems 
related to the autonomy of the system and the new agent, respectively. The first one, whose complexity is in PTIME, tests whether the system, under the normative system, can be autonomous in ensuring that the new agent can eventually recognise the norms, no matter how it plays. If such a level of autonomy is unachievable, we may consider the second subproblem, whose success suggests that if the new agent is autonomous (in moving in a smart way) then it can eventually recognise the norms. 
We show that the second subproblem is   PSPACE-complete.

\section{Partial Observation Multiagent Systems}\label{sec:pomas}

A multiagent system consists of a set of agents running in an environment~\cite{FHMVbook}. At each time, every agent takes a local action independently, and  the environment updates its state according to  agents' joint action. 
%
We assume that agents have only partial observations over the system states, 
because in most real-world systems, agents either do not have the capability of observing all the information (e.g., an autonomous car on the road can only observe those cars in the surrounding area by its sensors or cameras, etc) or are not supposed to observe private information of other agents (e.g., a car cannot observe the destinations of other cars, etc).

Let $\Ags$ be a finite set of agents and $\Prop$ be a finite set of atomic propositions. A finite  multiagent system 
is a tuple  $M=(S,\{Act_i\}_{i\in\Ags},\{L_i\}_{i\in\Ags},\{O_i\}_{i\in\Ags},I,T,\pi)$, where 
$S$ is a finite set of environment states, 
$Act_i$ is a finite set of local actions of agent $i\in \Ags$ such that $Act = Act_1 \times ...\times Act_n$ is a set of joint actions, 
$L_i:S\rightarrow \powerset{Act_i}\setminus \{\emptyset\}$ provides on every state a nonempty set of local actions that are available to agent $i$, 
$I\subseteq S$ is a nonempty set of initial states,
$T\subseteq S\times Act\times S$ is a transition relation such that for all $s\in S$ and $a\in Act$ there exists a state $s'$ such that $(s,a,s')\in T$,
$O_i:S\rightarrow \mathcal{O}$ is an observation function for each agent $i\in Agt$ such that $\mathcal{O}$ is a set of possible observations,  and
$\pi:S\rightarrow \mathcal{P}(\Prop)$ is an interpretation of the atomic propositions $\Prop$  at the states. 
%
We require that for all states $s_1,s_2\in S$ and $i\in\Ags$, $O_i(s_1)=O_i(s_2)$ implies $L_i(s_1)=L_i(s_2)$, i.e., an agent can distinguish two states with different sets of next available actions.  
%
%

\begin{example}\label{example:model}

We consider 
a business system 
with two sets of autonomous agents: the producer agents $P=\{p_1,...,p_n \}$, and consumer agents $C=\{c_1,...,c_m\}$. Let $\Ags=P\cup C$. Each producer agent $p_j\in P$ produces a specific kind of goods with limited quantity each time.
There can be more than one agents producing the same goods. 
We use $g_j\in G$ to denote the kind of goods that are produced by agent $p_j$, and $b_j \in \nat$ to denote the number of goods that can be produced at a time. 
Every consumer agent $c_i\in C$ 
has a designated job which needs a set of goods to complete. It is possible that more than one goods of a kind are needed. We use $r_i$ to denote the multiset of goods that are required by agent $c_i$. 


We use $rr_i\subseteq r_i$ to denote the multiset of remaining goods to be collected for $c_i$, $d_i \in G' = G\cup \{\bot\}$ to represent $c_i$'s current demand, and $t_i\in P' = P\cup \{\bot\}$ to represent the producer agent from whom $c_i$ is currently requesting goods. Every interaction of  agents occurs in two consecutive rounds, and we use $k\in \{1,2\}$ to denote the current round number. 

Because $g_j,b_j,r_i$ do not change their values in a system execution, we assume that they are fixed inputs of the system. The multiagent system $M$ has the state space as 
$$
S=\{1,2\}\times \Pi_{i\in \{1,...,m\}} \{(rr_i, d_i,t_i)~|~rr_i\subseteq r_i, d_i\in G', t_i\in P'\}
$$
where the first component $\{1,2\}$ is for the 
 round number.
The initial states are $I=\{1\}\times \Pi_{i\in \{1,...,m\}} \{(\emptyset, \bot,\bot)\}$.




The consumer agent $c_i$ has a set of  actions 
$Act_{c_i}=\{a_\bot\}\cup \{a_{p_j}~|~p_j\in P\}.$
Intuitively, $a_\bot$ means that an agent does nothing,
 and the action $a_{p_j}$ means that agent $c_i$ sends a request to producer $p_j$ for its goods. 
The producer agent $p_j$ has a set of  actions 
$Act_{p_j}=\{a_\bot\}\cup \{a_{\subC}~|~\subC\subseteq C, |\subC| \leq b_j\}$.
Intuitively, the action $a_{\subC}$ for $\subC$ a subset of agents represents that agent $p_j$ satisfies the requests from agents in $\subC$.

We use pseudocode to describe the transition relation. In the first round, i.e., $k=1$, it can be described as follows.  
\begin{enumerate}
\item[R1a.] all consumer agents $c_i$ do the following sequential steps: 
\begin{enumerate}
\item if $rr_i=\emptyset$ then we let $rr_i=r_i$. Intuitively, this represents that agent $c_i$'s job is repeated.
\item if $d_i=\bot$ then do the following: let $d_i\in rr_i$,
choose an agent $p_j$ such that $d_i=g_j$, and 
let $t_i=p_j$. Intuitively, 
if there is no current demand, then a new demand $d_i\in rr_i$ is generated, and $c_i$ sends a request to a producer agent $p_j$ who is producing goods $d_i$. 
\end{enumerate}
\item[R1b.] all producer agents $p_j$ execute action $a_\bot$, and let $k=2$. 
\end{enumerate}
In the second round, i.e., $k=2$, it can be described as follows. 
\begin{enumerate}
\item[R2a.] all producer agents $p_j$ do the following sequential steps: 
\begin{enumerate}
\item select a maximal subset $\subC$ of agents such that $\subC \subseteq \{c_i~|~t_i=p_j\}$ and $|\subC|\leq b_j$. Intuitively, from the existing requests, 
the producer agent $p_j$ selects a set  of  them according to its production capability. 
\item for all agents $c_i$ in $\subC$,  let $rr_i=rr_i\setminus \{g_j\}$ and $d_i=t_i=\bot$. Intuitively, if a demand $d_i$ is satisfied, then it is removed from $rr_i$ and we let $d_i=t_i=\bot$. 
\end{enumerate}
\item[R2b.] all consumer agents execute action $a_\bot$, and let $k=1$. 
\end{enumerate}%
%
%
%

We use ``$var=val$", for $var$ a variable and $val$ one of its values, to denote an atomic proposition. Then the labelling function $\pi$ can be defined naturally over the states. The observation $O_i$ will be discussed in Section \ref{sec:recognition}.

We provide a simple instantiation of the system\footnote{The instantiation is simply to ease the understanding of the definitions in Example~\ref{example:model} and \ref{example:dynamic}. The conclusions for the example system (i.e., Proposition~\ref{thm:insufficient}, \ref{thm:sufficient}, \ref{thm:failed}, \ref{thm:holds}, \ref{thm:nc3}) are based on the general definition. }. Let $n=2$, $G=\{g_1,g_2\}$, $b_1=b_2=1$ (two agents produce  goods one at each time), $m=3$, $r_1=\{g_1\}$, $r_2=\{g_2\}$ and $r_3=\{g_1, g_2\}$ (three consumers with the required goods). From the initial state  $s_0 = (1, 
(\emptyset, \bot, \bot), (\emptyset, \bot, \bot), (\emptyset, \bot, \bot))
$,  
we may have the following two  states such that $(s_0,(a_\bot,a_\bot,a_{p_1},a_{p_2},a_{p_1}),s_2')\in T$ and $(s_2',(a_{\{c_1\}},a_{\{c_2\}},a_\bot,a_\bot,a_\bot),s_1')\in T$: 

$
\begin{cases}
s_2' = (2,
( \{g_1\}, g_1, p_1), (\{g_2\}, g_2, p_2), (\{g_1,g_2\}, g_1, p_1)), and\\
s_1' = (1,
( \{\}, \bot, \bot), (\{\}, \bot, \bot), (\{g_1,g_2\}, g_1, p_1))
\end{cases}
$




\end{example}

\section{Dynamic Normative Systems}\label{sec:normsys}

The following is our new definition of normative systems. 
\begin{definition}
A dynamic normative system of a multiagent system\\ $M=(S,\{Act_i\}_{i\in\Ags},\{L_i\}_{i\in\Ags},\{O_i\}_{i\in\Ags},I,T,\pi)$ is a tuple $N_M=(Q,\delta_n,\delta_u,q_0)$ such that 
$Q$ is a set of normative states, 
$\delta_n: S\times Q  \rightarrow \powerset{Act}$ is a function specifying, for each environment state and each normative state, a set of joint actions that are disallowed, 
$\delta_u: Q \times S \rightarrow Q$ is a function specifying the update of normative states according to the changes of environment states, and 
$q_0$ is the initial normative state. 
\end{definition}

A (static) normative system in the literature can be seen as a special case of our definition where the only normative state is $q_0$.  
In such case, we have $Q=\{q_0\}$, $\delta_u(q_0,s)=q_0$ for all $s\in S$, and can therefore write the function $\delta_n$ as function $\delta: S \rightarrow \powerset{Act}$.
It is required that the function $\delta_n$ (and thus $\delta$) does not completely eliminate agents' joint actions, i.e., $\delta_n(s,q)\subset \Pi_{i\in\Ags}L_i(s)$ for all $s\in S$ and $q\in Q$. 


We give two dynamic normative systems. 
\begin{example}\label{example:dynamic} Let 
$M$
be the multiagent system given in Example~\ref{example:model}. Let $s_1$ and $s_2$ range over those environmental states such that $k=1$ and $k=2$, respectively. 

The normative system $N_M^1=(Q^1,\delta_n^1,\delta_u^1,q_0^1)$ is such that: 
\begin{itemize}
\item $Q^1=\Pi_{p_j\in P}\{1,...,m\}$, where each producer maintains a number indicating the consumer whose requirement must be 
satisfied in this normative state,  

%
\item $\delta_n^1(s_1,q)=\emptyset$, i.e., no joint actions are disallowed on $s_1$, and
$(a_{\subC_1},...,a_{\subC_n},a_\bot,...,a_\bot)\in\delta_n^1(s_2,(y_1,...,y_n))$ if there exists  $j\in \{1,...,n\}$ such that  $\subC_j\subseteq C$ and $c_{y_j}\notin \subC_j$. 
Intuitively, for producer agent $p_j$, an action $a_{\subC_j}$ is disallowed on the second round if $\subC_j$ does not contain the consumer $c_{y_j}$ who is needed to be satisfied in this round. 


\item $\delta_u^1(q,s_2)=q$ and $\delta_u^1((y_1,...,y_n),s_1)=(y_1',...,y_n')$ such that $y_j'=(y_j\mod m)+1$ for $j\in \{1,...,n\}$; intuitively, the normative state increments by 1 and loops forever. 
\item $q_0^1 = (1,...,n) $
, i.e., 
 producer agents $p_j$ start from $c_j$. 
\end{itemize}
For the instantiation in Example~\ref{example:model}, we have that 
\begin{itemize} 
\item $Q^1=\{1,2,3\}\times \{1,2,3\}$, $q_0^1 = (1,2)$, 
\item 
$(a_{\subC_1}, a_{\subC_2}, a_{\bot}, a_{\bot}, a_{\bot})\in 
\delta_n^1(s_2',(1,2))$ if either $\subC_1\in \{\emptyset,\{c_2\},\{c_3\},\{c_2,c_3\}\}$ or $\subC_2\in \{\emptyset,\{c_1\},\{c_3\},\{c_1,c_3\}\}$, 
\item $\delta_u^1((1,2), s_1)= (2,3)$, $\delta_u^1((2,3), s_1)= (3,1)$.  
\end{itemize}

We define  another normative system $N_M^2=(Q^2,\delta_n^2,\delta_u^2,q_0^2)$ by extending the number maintained by each producer into a first-in-first-out queue so that the ordering between consumers who have sent the requests matters. That is, we have $Q^2=\Pi_{p_j\in P}(\{\epsilon\}\cup\{ i_1...i_k~|~k\in \{1,...,m\}, i_x\in C \text{ for } 1\leq x\leq k\})$ where the symbol $\epsilon$ denotes an empty queue, and $q_0^2=\Pi_{p_j\in P}\{\epsilon\}$ which means that producers start from empty queues. The functions $\delta_n^2$ and $\delta_u^2$ can be  adapted from $N^1_M$, and details are omitted here. 


\end{example}


%
The following captures the result of applying a normative system on a multiagent system, which is essentially a product of these two systems.

\begin{definition} Let $M$ be a multiagent system and $N_M$ a normative system on $M$, the result of applying $N_M$ on $M$ is a Kripke structure $\M(N_M)=(S^\dagger,I^\dagger,T^\dagger,\pi^\dagger)$ such that 
\begin{itemize}
\item $S^\dagger=S\times Q$ is a set of states,
\item $I^\dagger = I \times \{q_0\}$ is a set of initial states,
\item $T^\dagger\subseteq S^\dagger\times S^\dagger$ is  such that, for any two states $(s_1,q_1)$ and $(s_2,q_2)$, we have $((s_1,q_1),(s_2,q_2))\in T^\dagger$ if and only if, 
(1) there exists an action $a\in Act$ such that $(s_1,a,s_2)\in T$ and $a\notin \delta_n(s_1,q_1)$, and 
(2) $q_2=\delta_u(q_1,s_2)$. 
Intuitively, the first condition specifies the enabling condition 
to transit from state $s_1$ to state $s_2$ by taking a joint action $a$ which 
is allowed in the 
normative state $q_1$. The second condition specifies that the transition relation needs to be consistent with the changes of normative states. 
\item $\pi^\dagger:S^\dagger\rightarrow \powerset{\Prop}$ is such that $\pi^\dagger((s,q))=\pi(s)$. 
\end{itemize}
\end{definition}

%
\begin{example}\label{example:kripke} 
For the instantiation, in the structure $\M(N_M^1)$, we have \\ $((s_0,(1,2)),(s_2',(1,2))), ((s_2',(1,2)),(s_1',(2,3)))\in T^\dagger$. The latter is because \\  $(s_2',(a_{\{c_1\}},a_{\{c_2\}},a_\bot,a_\bot,a_\bot),s_1')\in T$, $\{c_1\}\notin \{\emptyset,\{c_2\},\{c_3\},\{c_2,c_3\}\}$,\\ and $\{c_2\}\not\in \{\emptyset,\{c_1\},\{c_3\},\{c_1,c_3\}\}$. 

For $a=(a_{\{c_1\}},a_{\{c_3\}},a_\bot,a_\bot,a_\bot)$ and $s_1'' = (1,
( \{\}, \bot, \bot), (\{g_2\}, g_2, p_2), (\{g_1\}, \bot, \bot))$, we have 
$(s_2',a,s_1'')\in T$ but $((s_2',(1,2)),(s_1'',(2,3)))\not\in T^\dagger$. This is because, for  $p_2$, it is required to make $c_2$ as its current priority according to the normative state, and cannot choose to satisfy $c_3$ instead.
\end{example}

We remark that, the normative system, as many current formalisms, imposes hard constraints on the agents' behaviour. As stated in e.g., \cite{BvdTV2006}, social norms may be soft constraints that agents can choose to comply with or not. To accommodate soft social norms, we can redefine the function $\delta_n$ as $\delta_n: S\times Q  \times Act \rightarrow U$ to assign each joint action a cost utility for every agent, on each environment state and normative state. With this definition, norms become soft constraints: agents can choose to take destructive actions, but are encouraged to avoid them due to their high costs. 
The objective language to be introduced in the next section also needs to be upgraded accordingly to express properties related to the utilities. We leave such an extension as a future work. 


\section{Objective Language}

To specify agents' and the system's objectives, we use temporal logic CTL~\cite{clarkebook}
whose syntax is as follows. 
$$\phi~::=~p~|~\neg \phi~|~\phi_1\lor\phi_2~|~EX\phi~|~E(\phi_1\until\phi_2)~|~EG\phi$$
where $p\in\Prop$.
%
Intuitively, formula $EX\phi$ expresses that $\phi$ holds at some next state, $E(\phi_1\until\phi_2)$ expresses that on some path from current state, $\phi_1$ holds until $\phi_2$ becomes true, and $EG\phi$ expresses that on some path from current state, $\phi$ always holds. Other operators can be obtained as usual, e.g., 
$EF\phi\equiv E(True\until\phi)$, $AG\phi\equiv\neg E(True \until \neg \phi)$, $AF\phi=\neg EG \neg \phi$ etc.

A path in a Kripke structure $\M(N_M)$ is a sequence $s_0s_1...$ of states such that $(s_i,s_{i+1})\in T^\dagger$ for all $i\geq 0$. 
The  semantics of the language is given by a relation $\M(N_M),s\models \phi$ for $s\in S^\dagger$, 
which is defined inductively as follows~\cite{clarkebook}: 
\be 
\item $\M(N_M),s \models p$ if $p\in \pi^\dagger(s)$, 
\item $\M(N_M),s \models \neg \phi $ if  not  $\M(N_M), s\models  \phi$, 
\item $\M(N_M),s \models \phi_1 \lor \phi_2 $ if  $\M(N_M), s\models  \phi_1$ or  $\M(N_M), s\models  \phi_2$,
\item $\M(N_M),s \models EX \phi$ if  there exists a state $s'\in S^\dagger$ such that $(s,s')\in T^\dagger$ and  $\M(N_M), s'\models  \phi$, 
\item $\M(N_M),s \models E(\phi_1 \until \phi_2)$ if  there exists a  
path $s_0s_1\ldots $ and a number $n\geq 0$
such that $s_0=s$, $\M(N_M), s_k\models \phi_1$  for $0\leq k \leq n-1$ and  $\M(N_M), s_n\models  \phi_2$, 
\item $\M(N_M),s \models EG \phi$ if  there exists a   
path $s_0s_1\ldots $ such that $s_0=s$ and  $\M(N_M), s_k\models  \phi$ for all $k \geq 0$. 
\ee 

The verification problem, denoted as $\M(N_M)\models \phi$, is, given a multiagent system $M$, its associated normative system $N_M$, and an objective formula $\phi$, to decide whether $\M(N_M),s\models \phi$ for all $s\in I^\dagger$. 
The norm synthesis problem is, given a  system $M$ and an objective formula $\phi$, to decide the existence of a normative system $N_M$ such that $\M(N_M)\models \phi$. The norm recognition problem will be defined in Section~\ref{sec:recognition}. 
For the measurement of the complexity, we take the standard assumption that the sizes of the multiagent system and the normative system are measured with the number of states, and the size of the objective formula is measured with the number of operators. 




\begin{example}

For the system in Example~\ref{example:model}, interesting objectives expressed in CTL may include 
$$
\phi_1\equiv \bigwedge_{i\in C}\bigwedge_{j\in P} AG~(t_i=p_j\Rightarrow EF~d_i=\bot)
$$
which says that it is always the case that if there is a request from a consumer $c_i$ to a producer $p_j$ (i.e., $t_i=p_j$), then the request is possible to be satisfied eventually (i.e., $d_i=\bot$), and 
$$
\phi_2\equiv \bigwedge_{i\in C}\bigwedge_{j\in P} AG~(t_i=p_j\Rightarrow AF~d_i=\bot)
$$
which says that it is always the case that if there is a request from a consumer $c_i$ to a producer $p_j$, then on all the paths the request will eventually be satisfied. Both $\phi_1$ and $\phi_2$  are liveness objectives that are important for an ecosystem to guarantee that no agent can be starving forever. The objective $\phi_2$ is stronger than $\phi_1$, and their usefulness is application-dependent. 
%
The following proposition shows that static normative systems are insufficient to guarantee the satisfiability of the objectives in this ecosystem. 

\begin{propn}\label{thm:insufficient}
There exists an instance of a multiagent system $M$ such that, for all static normative systems $N_M$, we have that $\M(N_M)\not\models \phi_1\land\phi_2$. 
\end{propn}

The proof idea
is based on the following simple case. Assume that there are one producer $p_1$, such that $b_1=1$, and two consumers $c_1$ and $c_2$, such that $r_1=r_2=\{g_1\}$. There only exist the following three static normative systems which have different  restrictions on an environment state $s_2=(2,(\{g_1\},g_1,p_1),(\{g_1\},g_1,p_1))$: (Recall that $q_0$ is the only normative state in static normative systems.)
\begin{itemize}
\item $N^3_{M}$ is such that $\delta_n^3(s_2,q_0)=\{a_{\{c_1\}}\}$, i.e., $c_1$ is not satisfied. 
\item $N^4_{M}$ is such that $\delta_n^4(s_2,q_0)=\{a_{\{c_2\}}\}$, i.e., $c_2$ is not satisfied.
\item $N^5_{M}$ is such that $\delta_n^5(s_2,q_0)=\emptyset$, i.e., no restriction is imposed. 
\end{itemize}
We can see that 
$\M(N_M^h)\not\models \phi_1 \text{    for    } h \in \{3,4\}\text{ and }\M(N^5_{M})\not\models \phi_2.$
The former  is  because one of the agents is constantly excluded from being satisfied. For the latter, there exists an infinite path 
$s_0(s_2s_1)^\infty $
such that $s_0=(1,(\emptyset,\bot,\bot),(\emptyset,\bot,\bot))$ is an initial state, $s_2$ is given as above, 
and  $s_1=(1,(\emptyset,\bot,\bot),(\{g_1\},g_1,p_1))$ is the state on which consumer $c_i^1$'s requirement is satisfied. On this path, the requirement from $c_i^2$ is never satisfied.
{\em On the other hand}, for the dynamic normative systems in Example~\ref{example:dynamic}, all the consumers' requests can be satisfied, so we have the following conclusion. 
\begin{propn}~\label{thm:sufficient}
Given a 
system $M$ and a normative system $N_M^1$ or $N_M^2$, we have that $\M(N_M^h)\models \phi_1\land \phi_2 \text{ for }h\in \{1,2\}$. 
\end{propn}
\end{example}

The above example
suggests that, to achieve some objectives, we need {\it dynamic} normative systems to represent the changes of social norms under different circumstances. Then, another question may follow about the maximum number of normative states. The dynamic system could be uninteresting if the number of states can be infinite. Fortunately, in the next section, we show with the complexity result that, for objectives expressed with CTL formulas, in the worst case, an exponential number of normative states are needed.

\section{The Complexity of Norm Synthesis}

We have the following result for norm synthesis. 

\begin{theorem}\label{thm:synthesis}
The norm synthesis problem is EXPTIME-complete, with respect to the sizes of the system and the objective formula. 
\end{theorem}

\begin{proof}
We first show the upper bound: {\bf EXPTIME Membership}. 

From $M=(S,\{Act_i\}_{i\in\Ags},\{L_i\}_{i\in\Ags},\{O_i\}_{i\in\Ags},I,T,\pi)$, we define a \buchi\ tree automaton $A_{M}=(\Sigma,D,Q,\delta,q_0,Q)$ such that 
\be
\item $\Sigma=\mathcal{P}(V)\cup \{\bot\}$,  $Q=S\times \{\top,\vdash,\bot\}$, $q_0=(s,\top)$ for $s\in I$,
\item $D=\bigcup_{s\in S}\{1,...,|T(s)|\}$ where $T(s)=\{s'~|~\exists a\in Act: (s,a,s')\in T\}$, 
\item $\delta: Q\times \Sigma\times D\rightarrow 2^{Q^*}$ is defined as follows: for $s\in S$ and $k=|T(s)|$ with $T(s)=(s_1,...,s_k)$, we have 
(a) if $m\in \{\vdash, \bot\}$ then $\delta((s,m),\bot,k)=\{((s_1,\bot),...,(s_k,\bot))\}$, and 
(b) if $m\in \{\vdash,\top\}$, then we let 
$((s_1,y_1),...,(s_k,y_k))\in \delta((s,m),\pi(s),k)$ such that, there exists a nonempty set $B\subseteq \{1,...,k\}$ of  indices such that 
\be
\item $y_i=\top$, for all $i\in B$, and
\item $y_j=~\vdash$, for all $j\not\in B$ and $1\leq j\leq k$.
\ee
\ee 
Note that we use $\Fairness=Q$ to express that we only care about infinite paths. Moreover, the formula $\phi$ needs to be modified to reject those runs where $\bot$ is labeled on the states. This can be done by following the approach in~\cite{KV1996}. We still call the resulting formula $\phi$. 

Given a CTL formula $\phi$ and a set $D\subset \nat$ with a maximal element $k$, there exists a \buchi\ tree automaton $A_{D,\neg\phi}$ that accepts exactly all the tree models of $\neg\phi$ with branching degrees in $D$. By~\cite{VW1986}, the size of $A_{D,\neg\phi}$ is $O(2^{k\cdot |\phi|})$.
The norm synthesis problem over $M$ and $\phi $ for $\phi$ a CTL formula is equivalent to checking the emptiness of the product automaton $A_{M}\times A_{D,\neg\phi}$. The checking of emptiness of \buchi\ tree automaton can be done in quadratic time, so the norm synthesis problem 
for $\phi$ a CTL formula 
can be done in exponential time. 

Therefore, the norm synthesis problem over $M$ and $\phi$ can be done in exponential time with respect to $|S|$, $|\bigcup_{i\in\Ags}\Acts_i|$, and $|\phi|$. That is, it is in EXPTIME. 

We then show the lower bound: {\bf EXPTIME Hardness}

The lower bound is reduced from the problem of a linearly bounded alternating Turing machine (LBATM) accepting an empty input tape, which is known to be EXPTIME-complete~\cite{CKS1980}. Let $AT$ be an LBATM. A system $M(AT)$ of a single agent is constructed such that the agent moves on $\exists$ states and the environment moves on $\forall$ states. The normative system, applied on the agent's behaviour, may prune some branches of the system. 
We use an objective formula $\phi$ to express that the resulting system correctly implements several modification rules (which makes the resulting system moves as the $AT$ does) and all paths lead to  accepting states. Therefore, the norm synthesis problem on the system $M(AT)$ and the objective formula $\phi$ is equivalent to the acceptance of the automaton $AT$ on empty tape. That is, the complexity is EXPTIME hard. 

An alternating Turing machine $AT$ is a tuple $(Q,\Gamma,\delta,q_0,g)$ where $Q$ is a finite set of states, $\Gamma$ is a finite set of tape symbols including a blank symbol $\bot$, $\delta: Q\times\Gamma\rightarrow \mathcal{P}(Q\times \Gamma\times \{-1,+1\})$ is a transition function, $q_0\in Q$ is an initial state, $g:Q\rightarrow \{\forall,\exists,accept, reject\}$ specifies the type of each state. 
We use $b\in \Gamma$ to denote the blank symbol. 
The input $w$ to the machine is written on the tape. We use $w_i$ to denote the alphabet written on the $i$th cell of the tape. 

The size of a Turing machine is defined to be the size of space needed to record its transition relation, i.e., $2\times |\Gamma|^2\times |Q|^2$. 
An LBATM is an ATM which uses $m$ tape cells for a Turing machine description of size $m$. Let $L=\{1,...,m\}$. A configuration of the machine contains a state $q\in Q$, the header position $h\in L$, and the tape content $v\in \Gamma^*$. 
A configuration $c=(q,h,v)$ is accepting if $g(q)=accept$, or $g(q)=\forall$ and all successor configurations are accepting, or $g(q)=\exists$ and at least one of the successor configuration is accepting. 
The machine $AT$ accepts an empty tape if the initial configuration of $M$ (the state of $M$ is $q_0$, the head is at the left end of the tape, and all tape cells contain symbol $\bot$) is accepting, and to reject if the initial configuration is rejecting. It is known that the problem is EXPTIME-complete.

We construct a multiagent system $M$ with a single agent $i$. 
We have\\ $M=(S,\{Act_i\}_{i\in\Ags},\{L_i\}_{i\in\Ags},\{O_i\}_{i\in\Ags},I,T,\pi)$ where 
\be
\item $S=(Q\times L)\cup (Q\times L\times \Gamma)\cup (Q\times L\times L\times \Gamma)$,
\item $Act_i=\{a_2,a_3\}\cup\{a_{rcd}~|~r\in Q, c\in \Gamma, d\in \{-1,+1\}\}\cup \Sigma$, 
\item the function $L_i$ is defined as 
\be
\item 
$L_i((t_i,o_{q,h}))=\Sigma$,
\item 
$L_i((t_i,o_{q,h,b}))=\{a_2\}$ for $g(q)=\forall$,
\item  
$L_i((t_i,o_{q,h,b}))=\{a_{rcd}~|~r\in Q, c\in \Gamma, d\in \{-1,+1\}\}$ for $(r,c,d)\in\delta(q,b)$ and $g(q)=\exists$, and
\item 
$L_i((t_i,o_{q,h_1,h_2,c}))=\{a_3\}$.
\ee

\item the function $O_i$ is defined as follows:
$O_i((q,h))=o_{q,h}$,
$O_i((q,h,b))=o_{q,h,b}$,
$O_i((q,h_1,h_2,c))=o_{q,h_1,h_2,c}$.

\item $I=\{(q_0,1)\}$,
\item the transition relation $T$ is defined as follows:
\be
\item $((q,h),b,(q,h,b))\in T$ for $b\in\Gamma$.
\item $((q,h,b),a_2,(r,h+d,h,c))\in T$ for $(r,c,d)\in\delta(q,b)$ and $g(q)=\forall$.
\item $((q,h,b),a_{rcd},(r,h+d,h,c))\in T$ for $g(q)=\exists$.
\item $((r,h_1,h_2,c),a_3,(r,h_1))\in T$.
\ee

\ee
Intuitively, a transition $(r,c,d)\in\delta(q,b)$ is simulated by three consecutive transitions: 
\be
\item $((q,h),b, (q,h,b))$, where the agent guesses the correct symbol written in cell $h$. 
\item $((q,h,b),a,(r,h+d,h,c))$ such that 
if $a=a_2$ then the state $q$ is an $\forall$ state and it is the environment that moves according to  $\delta$, and 
if $a=a_{rcd}$ then the state $q$ is an $\exists$ state and it is the agent that moves according to $\delta$.
\item $((r,h+d,h,c),a_3,(r,h+d))$, where the system makes a deterministic transition.  
\ee
For the second transition, we let the environment move on $\forall$ states because all the successor states have to be explored, while let the agent move on $\exists$ states so that it allows the normative system to prune some branches. 

Let $V=\{Q_q~|~q\in Q\}\cup\{H_h,G_h~|~h\in L\}\cup\{R_b,W_b~|~b\in \Gamma\}\cup \{acc, k_2\}$ be a set of boolean variables. Intuitively, $Q_q$ represents that the current state is $q$, $H_h$ represents that the current header position is $h$, $G_h$ represents that the last header position is $h$, $R_b$ represents that the symbol on the current cell is $b$, and $W_b$ represents that the symbol written on the last header position is $b$. We define the labelling function $\pi$ as follows:
$\{Q_q,H_h\}\subseteq\pi((q,h))$,
$\{Q_q,H_h,R_b,k_2\}\subseteq \pi((q,h,b))$,
$\{Q_r,H_{h_1},G_{h_2},W_c\}\subseteq \pi((r,h_1,h_2,c))$,
$acc\in \pi((q,h))$ if $g(q)=accept$. Moreover, we let $\Fairness=\emptyset$.

We need the following formulas. 
\be
\item Formula 
$\phi_1(h)\equiv A((H_h\land k_2\Rightarrow R_{\bot})~\until \bigvee_{c\in\Gamma}(G_h\land W_c))$
expresses that the symbol on position $h$ is $\bot$ until it is modified. 
\item Formula  
$\phi_2(h,b)\equiv G_h\land W_b\Rightarrow A((H_h\land k_2\Rightarrow R_b)~\until \bigvee_{c\in\Gamma}(G_h\land W_c))$
expresses that once the symbol on position $h$ is modified into $b$, it will stay the same until the next modification occurs. 
\ee 
Then the formula to be model checked on the system is 
$$\phi=(\bigwedge_{h\in L} \phi_1(h) \land AG \bigwedge_{h\in L}\bigwedge_{b\in\Gamma}\phi_2(h,b))\land AF acc$$
Intuitively,  the norm synthesis problem over $M$ and $\phi$ is to determine the existence of a normative system, which by pruning the behaviour of the agent, can make the resulting system correctly implements the modification rules and all branches can be accepting.  Therefore, the norm synthesis problem is equivalent to the acceptance of the automaton $AT$ on empty tape. That is, the complexity is EXPTIME hard. 
\end{proof}

\section{Agent Recognition of Social Norms}\label{sec:recognition}

For a multiagent system to be autonomous without human intervention, it is important that it can maintain its functionality when new agents join or old agents leave. 
For a new agent to join and function well, it is essential that it is capable of recognising the social norms that are currently active. As stated in the previous sections, the agent has only partial observation over the system state, and is not supposed to observe the social norms. 
On the other hand, it is also unrealistic to assume that the agent does not know anything about the social norms of the system it is about to join. Agent is designed to have a set of prescribed capabilities and is usually supposed to work within some specific scenarios. Therefore, the actual situation can be that, the agent knows in prior that there are a set of possible normative systems, one of which is currently applied on the multiagent system. 
We remark that, assuming a set of normative systems does not weaken the generality of the setting, because Theorem 1 implies that there are a finite number of possible normative systems (subject to a bisimulation relation between Kripke structures). 
This situation naturally leads to the following two new problems: 
\begin{itemize} 
\item (${\bf NC_1}$) to determine whether the agent can always recognise which normative system is currently applied; and
\item (${\bf NC_2}$) to determine whether the agent can find a way to recognise which normative system is currently applied. 
\end{itemize}
The successful answer to the problem $NC_1$ implies the successful answer to the problem $NC_2$, but not vice versa. Intuitively, the successful answer of $NC_1$ implies a high-level autonomy of the system that  the new agent can be eventually incorporated into the system no matter how it behaves. We assume that once learned the social norms the new agent will behave accordingly. If such an autonomy of the system cannot be achieved, the successful answer of $NC_2$ implies a high-level autonomy of the agent that, by moving in a smart way, it can eventually recognise the social norms. 

We formalise the problems first. Let $\Psi$ be a set of possible normative systems defined on a multiagent system, 
$Path(\M(N))$ be the set of possible paths of the Kripke structure $\M(N)$ for $N\in \Psi$. 
We assign every normative system in $\Psi$ a distinct index, denoted as $ind(N)$. This index is attached to every path $\rho\in Path(\M(N))$, and let $ind(\rho)=ind(N)$. 

Let the new agent be $x$ such that $x\notin \Ags$ and its observation function be $O_x$. For any state $(s,q)\in S^\dagger$, we define a projection function $\widehat{(s,q)}=s$. So $\widehat{\rho}$ is the projection of a path $\rho$ of a Kripke structure to the associated multiagent system. We extend $O_x$ to the paths of Kripke structure $\M(N)$ as follows: 
$O_x(\rho s^\dagger)=O_x(\rho)\cdot O_x(\widehat{s^\dagger})$ for $\rho\in Path(\M(N)$ and $s^\dagger\in S^\dagger$. We have $O_x(\epsilon)=\epsilon$, which means when a path is empty, the observation is also empty. We also define its inverse 
 $O_x^{-1}$ which gives a sequence $o$ of observations, returns a set of possible paths $\rho$ on which agent $x$'s observations are $o$, i.e., 
$$O_x^{-1}(o)=\{\rho\in Path(\M(N))~|~O_x(\rho)=o, N\in\Psi\}.$$ 

W.l.o.g., we assume that $N_0\in \Psi$ is the active normative system. Let $\nat$ be the set of natural numbers, we have 
\begin{definition}\label{def:NC}
$NC_1$ problem is the existence of a number $k\in \nat$ such that 
for all paths $\rho\in Path(\M(N_0))$ such that $|\rho| \geq k$, we have that $\rho'\in O_x^{-1}(O_x(\rho))$ 
implies that $ind(\rho')=ind(\rho)$.

$NC_2$ problem is the existence of a path $\rho\in Path(\M(N_0))$ such that for all 
$\rho'\in O_x^{-1}(O_x(\rho))$ 
we have  $ind(\rho')=ind(\rho)$. 
\end{definition}
Intuitively, $NC_1$ states that as long as the path is long enough, the new agent can eventually know that the active normative system is $N_0$. That is, no matter how the new agent behaves, it can eventually recognise the current normative system. On the other hand, $NC_2$ states that such a path exists (but not necessarily for all paths). That is, to recognise the normative system, the new agent needs to move smartly.

\begin{example}
For the system in Example~\ref{example:model}, we assume a new consumer agent $c_v$ such that $v=m+1$. Also, we define  $O_{c_v}(s)=\{c_i~|~c_i\in C, t_i(s)=t_v(s)\neq \bot\}$ for all $s\in S$. 
Intuitively, the agent $c_v$ keeps track of the set of agents that are currently having the same request. Unfortunately, we have 

\begin{propn}\label{thm:failed}
There exists an instance of system $M$ such that under the set $\Psi=\{N_M^1,N_M^2\}$ of normative systems, both $NC_1$ and $NC_2$ are unsuccessful.
\end{propn}

This can be seen from a simple case where there are a single producer $p_1$ with $b_1=1$ and a set of consumers $C$ such that $r_i=\{g_1\}$ for all $c_i\in C$. For the initial state, every consumer sends its request to $p_1$, so $\{c_i~|~t_i=p_1\}=C$. For any path $\rho_1$ of $\M(N^1_M)$ and $\rho_2$ of $\M(N^2_M)$, we have $\widehat{\rho_1}=\widehat{\rho_2}=s_0s_2s_1^1s_2...s_1^ms_2s_1^vs_2s_1^1...$ where $s_0=(1,(\emptyset,\bot,\bot),...,(\emptyset,\bot,\bot))$, $s_2=(2,(\{g_1\},g_1,p_1),...,(\{g_1\},g_1,p_1))$, and $s_1^i$ is different with $s_2$ in $c_i$'s local state, e.g., $s_1^1=(1,(\{\},\bot,\bot),...,(\{g_1\},g_1,p_1))$. And therefore $O_{c_v}(\rho_1)=O_{c_v}(\rho_2) = \emptyset C(C\setminus\{c_1\})C(C\setminus\{c_2\})...$, i.e., the agent $c_v$'s observations are always the same\footnote{We reasonably assume that, for $N_M^1$, when a producer sees $c_v$, it will adjust its range in normative states from $\{1,...,m\}$ to $\{1,...,m,v\}$.}. 
That is, the new agent $c_v$ finds that for $\rho_1$, the single path on $K(N_M^1)$,  there are $\rho_2\in O_{c_v}^{-1}(O_{c_v}(\rho_1))$ and $ind(\rho_1)\neq ind(\rho_2)$. 
Therefore, neither $NC_1$ nor $NC_2$ can be successful in such a case. 
\end{example}

The reasons for the above result may come from either the insufficient capabilities of the agent or the designing of normative systems. We explain this in the following example. 
\begin{example}
First, consider that we increase the capabilities of the new agent by updating the rule $R2b$ in Section~\ref{sec:pomas}. 
\begin{itemize}
\item[R2b'.] the new agent $c_v$ may cancel its current request by letting $d_v=t_v=\bot$; all other consumer agents execute action $a_\bot$; and let $k=1$. 
\end{itemize}
With this upgraded capabilities of the new agent, the  $NC_2$ can be  successful. The intuition is that, by canceling and re-requesting for at least twice, the ordering of consumer agents whose requests are satisfied can be different in two normative systems: with $N_M^2$, there are other agents $c_i$ between $c_m$ and $c_v$, but with $N_M^1$, their requests are always satisfied consecutively. Note that, by its new capabilities, $c_v$ can always choose a producer agent which have more than 2 existing and future requests (Assuming that $n\ll m$, which is usual for a business ecosystem). 

\begin{propn}\label{thm:holds}
With the new rule R2b', the $NC_2$ problem is successful on system $M$ and the set $\Psi=\{N_M^1,N_M^2\}$.
\end{propn}

However, the $NC_1$ problem is still unsuccessful, because the agent $c_v$ may not move in such a smart way. 
For this, we replace $N_M^1$ with $N_M^6=(Q^1,\delta_n^1,\delta_u^6,q_0^1)$ such that 
\begin{itemize}
\item $\delta_u^6(q,s_2)=q$ and $\delta_u^6((y_1,...,y_n),s_1)=(y_1',...,y_n')$, s.t.
$y_j'=((y_j+1)\mod m)+1$ for $j\in \{1,...,n\}$.
Intuitively, the normative state increments by 2 (modulo m). 
\end{itemize}
\begin{propn}\label{thm:nc3}
Both $NC_1$ and $NC_2$ problems are successful on system $M$ and the set $\Psi=\{N_M^2,N_M^6\}$.
\end{propn}
\end{example}

\section{The Complexity of Norm Recognition}

The discussion in the last section clearly shows that, the two norm recognition problems are non-trivial. It is therefore useful to study if there exist efficient algorithms that can decide them automatically. 
In this section, we show a somewhat surprising result that the determination of $NC_1$ problem can be done in PTIME, while it is PSPACE-complete for $NC_2$ problem. 
%
Assume that the size of the set $\Psi$ is measured over both the number of normative systems and the number of normative states. We have the following conclusions.

\begin{theorem}\label{thm:NC_1}
The $NC_1$ problem can be decided in PTIME, with respect to the sizes of the system and the set $\Psi$. 
\end{theorem}

\begin{proof}
First of all, by its definition in Definition~\ref{def:NC}, the unsuccessful answer to an $NC_1$ instance is equivalent to the existence of two infinite paths $\rho\in Path(\M(N_0))$ and $\rho'\in O_x^{-1}(O_x(\rho))$ such that $ind(\rho')\neq ind(\rho)$. In the following, we give an algorithm to check such an existence. 

Recall that $x$ is the new agent. Let $\M(N)=(S_N^\dagger,I_N^\dagger,T_N^\dagger,\pi_N^\dagger)$ be the Kripke structure obtained by applying the normative system $N\in \Psi$ on the system $M$. We define the function $O_x^\dagger$ over the states $\bigcup_{N\in\Psi}S_{N}^\dagger $ by letting $O_x^\dagger((s,q))=O_x(s)$ for all $(s,q)\in S_N^\dagger$. Moreover, we extend  the $ind$ function to work with the states in $\bigcup_{N\in\Psi}S_{N}^\dagger $ by letting $ind((s,q))=ind(N)$ for $q$ being a normative state of $N$. 

A product system is constructed by synchronising the behaviour of two Kripke structures $\M(N_0)$ and $\M(N) $ for $N\in \Psi\setminus \{N_0\}$ such that the observations are always the same. Let $Prop'=\{safe\}$ be the set of atomic propositions.
Formally, it is the structure $M'=(S',I',T',\pi')$ such that 
\begin{itemize}
\item $S'=S_{N_0}^\dagger\times \bigcup_{N\in\Psi, N\neq N_0}S_{N}^\dagger$, 
\item $(s,t)\in I'$ if $s\in I_{N_0}^\dagger$ and $t\in \bigcup_{N\in\Psi, N\neq N_0}I_{N}^\dagger$ such that $O_x^\dagger(s)=O_x^\dagger(t)$. 
\item $((s,t),(s',t')) \in T'$ if and only if $(s,s')\in T_{N_0}^\dagger$, $(t,t')\in T_{N}^\dagger$ for some $N\in\Psi$ and $N\neq N_0$, and $O_x^\dagger(t)=O_x^\dagger(t')$ and 
\item $safe\in \pi((s,t))$ iff $ind(s)\neq ind(t)$. 
\end{itemize}
Intuitively, the structure consists of two components: one component moves according to the Kripke structure $\M(N_0)$ (that is, the currently active normative system), and the other moves according to other Kripke structures $\M(N)$ by matching the observations of the new agent $x$. 
Therefore, we have the equivalence of the following two statements: 
\begin{itemize}
\item the existence of two infinite paths $\rho\in Path(\M(N_0))$ and $\rho'\in O_x^{-1}(O_x(\rho))$ such that $ind(\rho')\neq ind(\rho)$; 
\item the existence of an infinite path in the structure $M'$ such that all states on the path are labelled with $safe$. 
\end{itemize}
Then, the existence of an infinite path where all states are labelled with an atomic proposition can be reduced to 1) the removal of all states (and their related transitions) not labelled with the atomic proposition and then 2) the checking of reachable strongly connected components (SCCs). 

For the complexity, we notice that $M'$ is polynomial over $M$ and $\Psi$, and the checking of reachable SCCs can be done in PTIME by the Tarjan's algorithm~\cite{Tarjan1972}.  
\end{proof}


\begin{theorem}\label{thm:NC_2}
The $NC_2$ problem is PSPACE-complete, with respect to the sizes of the system and the set $\Psi$. 
\end{theorem}

\begin{proof}

We first show the upper bound: {\bf PSPACE Membership} 

The upper bound is obtained by having a nondeterministic algorithm which takes a polynomial size of space, i.e., it is in NPSPACE=PSPACE. 

First of all, by its definition in Definition~\ref{def:NC}, the successful answer to an $NC_2$ instance is equivalent to the existence of a finite paths $\rho\in Path(\M(N_0))$ such that  all paths $\rho'\in O_x^{-1}(O_x(\rho))$ of the same observation belongs to $\M(N_0)$, i.e., $ind(\rho')= ind(\rho)$. The idea of our algorithm is as follows. It starts by guessing a set of initial states of the structures $\{\M(N)~|~N\in\Psi\}$ on which agent $x$ has the same observation. It then continuously guesses the next set of states such that they are reachable in one step from some state in the current set and on which agent $x$ has the same observation. If this guess can be done infinitely then the $NC_2$ problem is successful. 
This infinite number of guesses can be achieved with a finite number of guesses, by adapting the approach of LTL model checking~\cite{clarkebook}.

Let $Prop''=\{goal\}$ be the set of atomic propositions. We define the structure $M''=(S'',I'',T'',\pi'')$ such that 
\begin{itemize}
\item $S''=S_{N_0}^\dagger\times \powerset{\bigcup_{N\in\Psi}S_{N}^\dagger}$, 
\item $(s,P)\in I''$ if $s\in I_{N_0}^\dagger$ and $P\subseteq \bigcup_{N\in\Psi}I_{N}^\dagger$ such that $t\in P$ iff $O_x^\dagger(s)=O_x^\dagger(t)$, 
\item $((s,P),(t,Q)) \in T''$ if and only if $(s,t)\in T_{N_0}^\dagger$ and $Q=\{t'~|~\exists s'\in P\exists N\in \Psi: (s',t')\in T_{N}^\dagger, O_x^\dagger(t')=O_x^\dagger(t)\}$, and
\item $goal\in \pi((s,P))$ iff for all states $s\in P$ we have $ind(s)=ind(N_0)$. 
\end{itemize}
Intuitively, each path of the structure $M''$ represents a path in $K(N_0)$ (in the first component) together with the set of paths with the same observation for agent $x$ (in the second component). Therefore, we can have the equivalence of the following two statements: 
\begin{itemize}
\item the existence of a finite paths $\rho\in Path(\M(N_0))$ such that  all paths $\rho'\in O_x^{-1}(O_x(\rho))$ of the same observation belongs to $\M(N_0)$, i.e., $ind(\rho')= ind(\rho)$; 
\item in structure $M''$, the existence of an initial state such that it can reach some state satisfying $goal$. 
\end{itemize}

 For the complexity of the algorithm, we note that although the system $M''$ is of exponential size,  the reachability can be done on-the-fly by using a polynomial size of space. 

We then show the lower bound: {\bf PSPACE Hardness}

\newcommand{\loopstate}{{loop}}

It is obtained by a reduction from the problem of deciding if, for a given nondeterministic finite state automaton $A$ over an alphabet $\Sigma$, the language $L(A)$ is  equivalent to the universal language $\Sigma^*$. Let $A=(Q,q_0,\delta,F)$ be an NFA such that $Q$ is a set of states, $q_0\in Q$ is an initial state, $\delta: Q\times \Sigma\rightarrow \powerset{Q}$ is a transition function, and $F\subseteq Q$ is a set of final states. We construct a system $M(A)$ which consists of two subsystems, one of them simulates the behaviour of $A$ and the other simulates the behaviour of the language $\Sigma^*$. The subsystems are reachable from an initial state $s_0$ by two actions $a_1$ and $a_2$ respectively. Let $\Sigma_1=\Sigma\cup\{\bot\}$ such that $\bot\notin\Sigma$ is a  symbol. Formally, we have that $M(A)=(S,\{Act_i\}_{i\in\Ags},\{L_i\}_{i\in\Ags},\{O_i\}_{i\in\Ags},I,T,\pi)$ is a single-agent system such that 
\begin{itemize}
\item $\Ags = \{x\}$,
\item $S=S^1\cup S^2\cup \{s_0,s_\loopstate\}$ with $S^1=\Sigma_1\times Q$ and $S^2=\{s_a~|~a\in \Sigma_1\}$, 
\item $Act_x=\Sigma\cup \{a_1,a_2\}$,
\item $L_x(s_0)=\{a_1,a_2\}$ and $L_x(s)=\Sigma$ for $s\in S\setminus \{s_0\}$,
\item $O_x(s_0)=O_x(s_\loopstate)=\bot$,  $O_x((a,q))=a$ for $(a,q)\in \Sigma_1\times Q$, and $O_x(s_a)=a$ for $a\in\Sigma_1$, where $\mathcal{O}=\Sigma_1$, 
\item $I(s_0)=1$, 
\item the transition relation $T$ consists of the following five sets of transitions: 
\begin{itemize} 
\item $\{(s_0,a_1,(\bot,q_0)),(s_0,a_2,s_\bot)\}$; intuitively, from the initial state $s_0$, it can transit into the subsystem $M_1(A)$ by taking action $a_1$ or the subsystem $M_2(A)$ by taking action $a_2$,
\item $\{((a,q),a_1,(a_1,q_1))~|~q,q_1\in Q, a\in\Sigma_1, a_1\in \Sigma, q_1\in \delta(q,a_1)\}$; intuitively, the subsystem $M_1(A)$ follows the behaviour of the automaton $A$,
\item $ \{((a,q),a_1,s_\loopstate)~|~q\in Q, a,a_1\in \Sigma, \delta(q,a_1)=\emptyset \}$; intuitively, in $M_1(A)$, all illegal actions take the system state into a designated state $s_\loopstate$,
\item $\{(s_\loopstate,a,s_\loopstate)~|~a\in \Sigma\}$; intuitively, the state $s_\loopstate$ is a loop state for all actions,  
\item $ \{s_a,a_1,s_{a_1})~|~a\in\Sigma_1, a_1\in \Sigma\}$; intuitively, the subsystem $M_2(A)$ simulates the language $\Sigma^*$, and 
\end{itemize}
\item $\pi$ will not be used. 
\end{itemize}
The new agent $x$ is the only agent of the system. 
On the system $M(A)$, we have two normative systems whose only difference is on the state $s_0$: $N_0$ disallows action $a_1$ and $N_1$ disallows action $a_2$.  
Formally, $N_0=(\{t_0\},\delta_n^0,\delta_u^0,t_0)$ such that  
\begin{itemize}
\item $\delta_n^0(s_0,t_0)=\{a_1\}$, $\delta_n^0(s,t_0)=\emptyset$ for all $s\in S\setminus \{s_0\}$, and 
\item $\delta_u^0(t_0,s)=t_0$ for all $s\in S$. 
\end{itemize}
and $N_1=(\{t_1\},\delta_n^1,\delta_u^1,t_1)$ such that  
\begin{itemize}
\item $\delta_n^1(s_0,t_1)=\{a_2\}$, $\delta_n^1(s,t_1)=\emptyset$ for all $s\in S\setminus \{s_0\}$, and 
\item $\delta_u^1(t_1,s)=t_0$ for all $s\in S$. 
\end{itemize}

Now we show that the universality of the NFA $A$ is equivalent to the unsuccessful answer to the $NC_2$ problem on $M(A)$ and $\Psi=\{N_0,N_1\}$. 

($\Rightarrow$) Assume that the automaton $A$ is universal. Then for all paths $\rho\in \M(N_0)$, there exists another path $\rho'\in\M(N_1)$ such that $O_x(\rho)=O_x(\rho')$. The inverse statement of the latter is that, there exists a finite path $\rho\in \M(N_0)$ such that there exists no $\rho'\in\M(N_1)$ such that $O_x(\rho)=O_x(\rho')$. The latter means that, all paths $\rho'$ with $O_x(\rho)=O_x(\rho')$ are in $\M(N_0)$, which is the statement of $NC_2$ problem.  

($\Leftarrow$) Assume that we have  the unsuccessful answer to the $NC_2$ problem on $M(A)$ and $\Psi=\{N_0,N_1\}$. Then by the definition, it means that for all paths $\rho\in \M(N_0)$, there exists another path $\rho'\in\M(N_1)$ such that $O_x(\rho)=O_x(\rho')$. The latter is equivalent to the fact that the automaton $A$ is universal. 

\end{proof}

\section{Related Work}

Normative multiagent systems have attracted many research interests in recent years, see e.g., \cite{BvdTV2006,CAB2011} for  comprehensive reviews of the area. Here we can only review some closely related work. 


\noindent {\bf Norm synthesis for static normative systems.} As stated, 
most current formalisms of normative systems are static. \cite{ST1995} shows that this norm synthesis problem is NP-complete. \cite{CR2009} proposes a  norm synthesis algorithm in declarative planning domains for reachability objectives, and \cite{MSRWV2013} considers the on-line synthesis of norms. \cite{BD2011} considers the norm synthesis problem by conditioning over agents' preferences, expresses as pairs of LTL formula and utility, and a normative behaviour function. 


\noindent{\bf Changes of normative system.}
\cite{KDM2014} represents the norms as a set of atomic propositions and then employs a language to specify the update of norms. Although the updates are parameterised over actions, no considerations are taken to investigate, by either verification or norm synthesis, whether the normative system can be imposed to coordinate agents' behaviour to secure the objectives of the system. 


\noindent{\bf Norm recognition.}
Norm recognition can be related to the norm learning problem, which employs various approaches, such as data mining \cite{SCPP2013} and sampling and parsing \cite{OM2013,CSMO2015}, for the agent to learn social norms by observing other agents' behaviour. On the other hand, our norm recognition problems are based on formal verification, aiming to decide whether the agents are designed well so that they can recognise the current normative system from a set of possible ones. We also study the complexity of them. 


\noindent{\bf Application of social norms}
%
%
Social norms are to regulate the behaviour of the stakeholders in a system, including sociotechnical system~\cite{CS2016} which has both humans and computers. They are used to represent the commitments (by e.g., business contracts, etc) between humans and organisations. The dynamic norms of this paper can be useful to model more realistic scenarios in which commitments may be changed with the environmental changes.  

\section{Conclusions}

In the paper, we first present a novel definition of normative systems, by arguing with an example that it can be a necessity to have multiple normative states. We study the complexity of two autonomy issues related to normative systems. The decidability  (precisely, EXPTIME-complete) of norm synthesis is an encouraging result, suggesting that the maximum number of normative states is bounded for CTL objectives. For the two norm recognition subproblems, one of them is, surprisingly, in PTIME and the other is PSPACE-complete. Because the first one suggests a better level of autonomy,
to see if an agent can recognise the social norms, we can deploy a PTIME algorithm first. 
If it fails, we may apply a PSPACE algorithm to check the weaker autonomy. 

\section*{Acknowledgements}

A conference version of this work is accepted by the 25th International Joint Conference on Artificial Intelligence (IJCAI-16). We thank the anonymous reviewers from the IJCAI-16 for their helpful comments and suggestions. 


\commentout{

\newpage

\appendix

\section{Proofs of Theorem~\ref{thm:synthesis}}

\section{Proofs of Theorem~\ref{thm:NC_1}}

\section{Proofs of Theorem~\ref{thm:NC_2}}

}

\end{document}